%% file: paper.tex
\documentclass{llncs}
\usepackage{makeidx}  
\usepackage{etoolbox} 
\makeatletter
\patchcmd{\ps@headings}{\rlap{\thepage}}{}{}{}
\patchcmd{\ps@headings}{\llap{\thepage}}{}{}{}
\makeatother
\input{packages}
\input{commands}

\begin{document}
\frontmatter          
\pagestyle{headings}  
%
\mainmatter              
\title{Revisiting the Minimum Constraint Removal Problem in Mobile Robotics}
%
%
\author{Antony Thomas \and Fulvio Mastrogiovanni \and Marco Baglietto}
\authorrunning{Antony Thomas and Fulvio Mastrogiovanni and Marco Baglietto} 
%

%
\institute{Department of Informatics, Bioengineering, Robotics, and Systems Engineering, University of Genoa, Via All'Opera Pia 13, 16145, Genoa, Italy.\\
\email{antony.thomas@dibris.unige.it, fulvio.mastrogiovanni@unige.it, marco.baglietto@unige.it}}

\maketitle              

\begin{abstract}
The \textit{minimum constraint removal problem} seeks to find the minimum number of constraints, i.e., obstacles, that need to be removed to connect a start to a goal location with a collision-free path. 
This problem is NP-hard and has been studied in robotics, wireless sensing, and computational geometry. 
This work contributes to the existing literature by presenting and discussing two results. 
The first result shows that the minimum constraint removal is NP-hard for simply connected obstacles where each obstacle intersects a constant number of other obstacles. 
The second result demonstrates that for $n$ simply connected obstacles in the plane, instances of the minimum constraint removal problem with minimum removable obstacles lower than $(n+1)/3$ can be solved in polynomial time. 
This result is also empirically validated using several instances of randomly sampled axis-parallel rectangles.
\end{abstract}

\section{Introduction}
\input{introduction}
\section{Related Work}
\input{related}
\section{Hardness Result for Fixed Obstacle Intersections}
\label{sec:problem}
\input{problem}
\section{A Sub-class in P}
\label{sec:fixed}
\input{fixed}
\section{Conclusion}
We have shown that MCR is NP-hard for instances in which no obstacle intersects more than a constant number of other obstacles. 
This answers an open question posed in~\cite{hauser2014IJRR} on whether MCR is in $P$ on problems in which each obstacle is connected and intersects no more than $k$ obstacles. 
The tractability of such instances is also questioned in~\cite{erickson2013AAAI} wherein the authors state without proof that the problem is trivial when the obstacles are allowed to intersect only pairwise. 
Eiben \textit{et al.}~\cite{eiben2018AAAI} showed that instances in which no more than two obstacles overlap at the same point cannot be solved in time $2^{o(\sqrt{N})}$, where $N = O(n^2)$ and $n$ denotes the number of obstacles in the plane. 
Instances restricted to the pairwise intersection are a sub-class of the 2-obstacle overlap MCR. 
This suggests a future line of study: is MCR in $P$ for instances with only $k-$wise allowed intersections, that is only groups of $k$ obstacles allowed to intersect, and $k$ held fixed? 

We have also shown that for $n$ simply connected obstacles in the plane, instances of MCR with $|S^{\star}|< (n+1)/3$ can be solved in polynomial time. 
A consequence of these results is that the exact graph search could be employed if it is known beforehand that $|S^{\star}|< (n+1)/3$ holds. 
However, algorithms that compute tight bounds are currently unknown. 
The existence of polynomial time algorithms for MCR computing tight upper bounds is left as an open theoretical question.
\bibliographystyle{splncs03.bst}
\bibliography{/home/antony/research_genoa/References/References}
\end{document}

%% file: packages.tex
\usepackage{url}
\usepackage{makeidx}         
\usepackage{graphicx}        
\usepackage{multicol}        
\usepackage[bottom]{footmisc}

\usepackage{amsmath,bm,amsfonts,amssymb}

\usepackage{commath}      
\usepackage{color,xcolor,ucs}
\usepackage{adjustbox}
\usepackage{subfig}
\usepackage[font=small,labelfont=bf]{caption}

\usepackage{floatrow}
\usepackage{tabularx}
\usepackage{float}
\usepackage{booktabs}
\usepackage{cite}
\usepackage{xr-hyper}
\usepackage{wrapfig}

\usepackage{algorithm}
\usepackage{algpseudocode}
\usepackage{multirow}
\usepackage{listings}
\usepackage{lipsum}
\usepackage{textcomp} 
\usepackage{tikz} 
\usepackage{tikzscale}
\usepackage{array}
\usepackage[nocomma]{optidef}  
\usepackage{pifont}

%% file: commands.tex
\makeatletter
\newcommand{\clearsubcaptcounter}{\setcounter{sub\@captype}{0}}
\makeatother

\newcolumntype{C}[1]{>{\centering\arraybackslash}p{#1}}

%% file: introduction.tex
The classical robot motion planning problem seeks to find a collision-free path from a start location to a goal location in the presence of fixed obstacles. 
However, in many situations, a collision-free path may not exist, but obstacles could be moved. 
For example, in a typical in-home setting, obstacles could correspond to chairs that could be re-positioned to allow for a feasible path, or they could be objects located inside a kitchen shelf that could be rearranged to reach a desired object. 
If no collision-free path exists, and if we allow for paths to intersect (collide with) obstacles, then one may naturally ask what is the minimum number of obstacles that are to be removed so that there exists a collision-free path from start to goal.
This problem, introduced by Hauser~\cite{hauser2014IJRR} is called the Minimum Constraint Removal (MCR) problem, and it is related to finding the fewest constraints (obstacles) that must be removed from the environment to connect a start and a goal location with a free path. 

Hauser showed that MCR is NP-hard by reduction from the SET-COVER problem. 
Erickson and Lavalle~\cite{erickson2013AAAI} performed a reduction from the maximum satisfiability problem for binary Horn clauses, proving that MCR is NP-hard for convex polygonal obstacles. 
The problem is also NP-hard when the obstacles are line-segments~\cite{alt2011EuroCG,yang2012PHD}. 
As a result, approximate solutions are often sought after~\cite{krontiris2017AR}. 
Bandyapadhyay \textit{et al.}~\cite{bandyapadhyay2020CG} showed that there exists an $O(\sqrt{n})$ approximation for rectilinear polygons and disks (with $n$ denoting the number of polygons or disks). 
MCR as translatd to computing the minimum color in an edge-colored graph is shown to have an $O(n^{2/3})$ polynomial time approximate algorithmic solution~\cite{kumar2019AEA}, where $n$ denotes the vertices on the translated graph.

Minimum constraint removal has also been studied under the name \textit{barrier resilience} by those who work in sensor networks ~\cite{kumar2005ICMCN,bereg2009ALGOSENSORS,chan2014TCS}. 
This problem considers an arrangement of different regions in a 2D plane, where each region represents the \textit{detection region} of a sensor. 
Given a start and a goal point, the resilience of the network is then measured by the minimum number of distinct sensor detection regions intersected, that is the minimum number of sensors that need to fail to create a sensor-avoiding path. 
Chan and Kirkpatrick~\cite{chan2014TCS} reduce the barrier resilience problem to the minimum-color path problem~\cite{yuan2005INFOCOM}, and compute a $1.5$-approximation for unit disks (sensor regions) which is an improvement over the $2$-approximation developed in~\cite{bereg2009ALGOSENSORS}. 
The problem is also shown to be NP-hard even for axis-aligned rectangles with an aspect ratio close to one~\cite{korman2018CG}.

In this paper, we demonstrate the following:
\begin{enumerate}
\item MCR is NP-hard for simply connected obstacles, where each obstacle intersects a constant number of other obstacles. Such a claim addresses an open question presented in~\cite{hauser2014IJRR} --- Is MCR in P on problems in which each obstacle is connected and intersects no more than $k$ obstacles, with $k$ held constant? 
\item For $n$ simply connected obstacles in the plane, instances of MCR with minimum removable obstacles lower than $(n+1)/3$ can be solved in polynomial time.
\end{enumerate}

The result in (1) also augments the findings presented in~\cite{eiben2018AAAI}, wherein MCR with each obstacle intersecting a constant number of other obstacles is shown to run in $2^{O(\sqrt{N})}$, where $N$ is the number of regions formed by the intersections of the $n$ obstacles ($N= O(n^2)$). 
The existence of a sub-exponential time (in $N$) algorithm is proved by considering line segments as obstacles and performing a reduction from the MAXIMUM NEGATIVE 2-SATISFIABILITY problem which is NP-hard. 
In contrast, we provide general proof for simply connected obstacles by considering different paths induced by the intersecting obstacles. 

If it is known beforehand that the optimal solution to MCR, which we refer to as $S^{\star}$, is lower than $(n+1)/3$, that is, $|S^{\star}|< (n+1)/3$, the result in (2) implies that an exact algorithm could be accorded more priority than an approximate one. 
Proof and further discussions on this result are provided in Section~\ref{sec:fixed}.

%% file: related.tex
Minimum constraint removal is closely related to different areas of research within the robotics literature. MCR finds the minimal set of obstacles that render motion planning infeasible. Disconnection proving or infeasibility proofs in motion planning~\cite{basch2001ICRA,zhang2008WAFR,li2021RSS} that guarantee that no path can be found is thus a related class of problem. Another closely related line of work is the Minimum Constraint Displacement (MCD) motion planning~\cite{hauser2013RSS,thomas2022IAS,thomas2023ICRA}. While MCR identifies the minimum number of obstacles that must be removed to yield a feasible path, MCD computes the minimum amount by which obstacles must be displaced to enable a feasible motion plan. 

In Navigation Among Movable Obstacles (NAMO)~\cite{stilman2005IJHR,nieuwenhuisen2008WAFR,van2009WAFR} class of problems, a robot manipulates different objects in the plane to create space for its motion. However, such approaches could move many objects (not necessarily the minimal set) and the corresponding  displacements may be arbitrary. Another related area is the field of task and motion planning~\cite{kaelbling2013IJRR,srivastava2014ICRA,dantam2016RSS,garrett2018IJRR,thomas2021RAS}. In such domains, additional symbolic actions are generated to ensure motion planning feasibility of the required task. Manipulation among clutter or rearrangement planning in clutter~\cite{stilman2007ICRA,dogar2011RSS,krontiris2015RSS,karami2021AIIA} also pose related challenges since this requires moving aside a set of obstacle (obstacles that hinder the task) or placing them at selected locations. 

%% file: problem.tex
In this Section, we prove that instances of MCR where each obstacle intersects up to a fixed number of other obstacles are NP-hard. 

It is well known that MCR problems can be converted to a graph search problem~\cite{erickson2013AAAI,hauser2014IJRR,gorbenko2015AIP,eiben2018ICALP}. 
To attain this transformation, the configuration space is partitioned along the obstacle boundaries, and a graph is constructed such that the vertices are the faces of the obstacle arrangements and the edges connect neighboring faces. 
The start $s$ and goal $g$ locations are augmented by adding the appropriate edges to obtain a graph $G=(V, E)$, where $V$ is the set of vertices, and $E$ is a collection of pairs of vertices $(v, v')$, where $v \neq v'$, or edges. 
An example is given in Fig.~\ref{fig:C}. 
It is noteworthy that when the obstacles do not intersect each other, then each vertex $v \in V$ corresponds to a single obstacle, and any path passing through $k$ such vertices implies that the $k$ corresponding obstacles must be removed to obtain a feasible path\footnote{MCR with disjoint and simply connected obstacles is in P~\cite{hauser2014IJRR}.}. 
When obstacles are allowed to intersect, each vertex could correspond to more than one obstacle. 
For each vertex $v$, we thus define the intersection function $I(v)$ as the set of intersecting obstacles that are represented by the vertex, that is, a path $P_G$ induced on $G$ that passes through vertex $v$ intersects all the obstacles in $I(v)$. 
Consequently, for a path $P$ in the configuration space from $s$ to $g$ to be feasible, the union of intersection functions for all vertices $v$ lying on $P_G$ need to be computed. 
This gives the set of obstacles that should be moved to make the path $P$ feasible, and it will be called the removable set $R_P$ for the path $P$. 
Similarly, the removable set of a path $P$ from $s$ to a generic vertex $v$ can be denoted as $R_P(v)$\footnote{In the following, we will omit the subscript $P$ because the reference to a specific path $P$ will be implicit.}.   

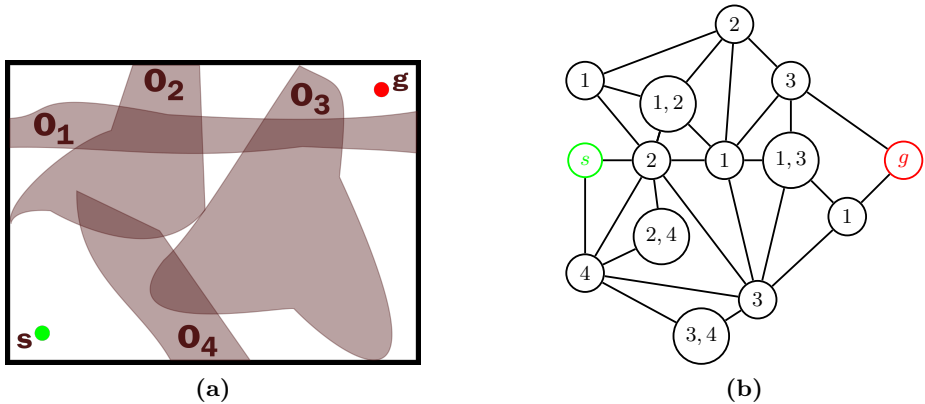
\begin{figure}[t]
\subfloat[]{\includegraphics[scale=0.55]{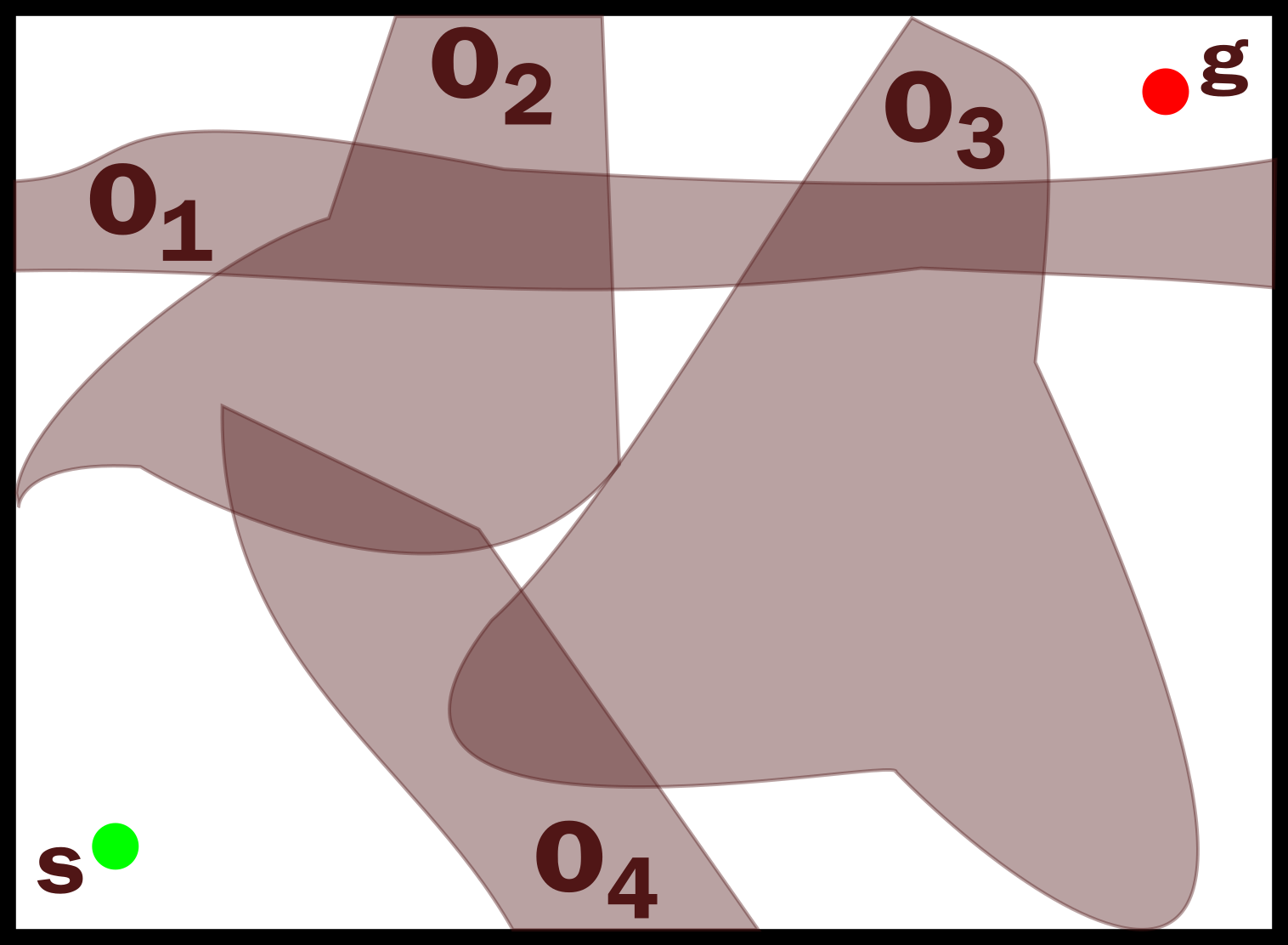}\label{fig:C1}}\hfill
\subfloat[]{\resizebox{0.39\textwidth}{!}{\begin{tikzpicture}[node distance={12mm}, thick, main/.style = {draw, circle}] 
\node[main] (1) [color=green] {$s$}; 
\node[main] (2) [above of=1] {$1$}; 
\node[main] (3) [below of=2, xshift=1cm] {$2$}; 
\node[main] (4) [below right of=3, yshift=-0.3cm, xshift=-0.7cm] {$2,4$}; 
\node[main] (5) [below of=3, yshift=-0.5cm, xshift=-1cm] {$4$}; 
\node[main] (6) [above right of=3, xshift=-0.6cm] {$1,2$}; 
\node[main] (7) [above of=6, xshift=1cm] {$2$}; 
\node[main] (8) [below right of=6] {$1$}; 
\node[main] (9) [below of=8, yshift=-0.9cm, xshift=0.5cm] {$3$};
\node[main] (10) [below left of=9, yshift=0.3cm] {$3,4$};
\node[main] (11) [below right of=7] {$3$};
\node[main] (12) [below of=11] {$1,3$};
\node[main] (13) [below right of=12] {$1$};
\node[main] (14) [color=red, above right of=13] {$g$};
\draw (1) -- (3); 
\draw (1) -- (5); 
\draw (2) -- (7);
\draw (2) -- (3);
\draw (2) -- (6);
\draw (3) -- (4);
\draw (3) -- (6);
\draw (3) -- (5);
\draw (3) -- (8);
\draw (3) -- (9);
\draw (3) -- (6);
\draw (5) -- (9);
\draw (5) -- (10);
\draw (10) -- (9);
\draw (9) -- (8);
\draw (9) -- (12);
\draw (9) -- (13);
\draw (8) -- (6);
\draw (8) -- (12);
\draw (12) -- (13);
\draw (6) -- (7);
\draw (7) -- (11);
\draw (11) -- (12);
\draw (11) -- (14);
\draw (14) -- (13);
\draw (5) -- (4);
\draw (7) -- (8);
\draw (11) -- (8);
\end{tikzpicture}}\label{fig:C2}}\hfill
\caption{
(a) An instance of MCR with 4 obstacles. 
Start and goal locations are shown in green and red, respectively. 
(b) The MCR problem transformed into a graph. 
Vertices represent obstacle sets (obstacle numbers inside the nodes) that form 2D connected regions based on the partition along obstacle boundaries.}
\label{fig:C}
\end{figure}
 
Given a vertex $v$, a path $P$, and the corresponding path on a graph $P_G$, MCR on a graph can be solved exactly by expanding the graph in the increasing order of the cardinality $|R(v)|$ of the removable set. 
To this end, each state is encoded as a pair $(v, R(v))$, representing the vertex and the corresponding removable set. 
This approach is complete and optimal~\cite{hauser2014IJRR}. 
When the obstacles are disjoint and simply connected, the exact search corresponds to finding the shortest path on a graph with a unit cost associated with each vertex. 
In this case, we can compute the minimum $|R(v)|$ path by employing the Dijkstra's algorithm. 
However, when the obstacles are overlapping, paths could \textit{enter} as well as \textit{exit} an obstacle more than once. 
This makes the removable set $R(v)$ path-dependent since an obstacle could potentially belong to the intersection function $I(v)$ of all the vertices, causing NP-hardness. 

Let us now consider the case in which each obstacle is allowed to intersect up to $k$ other obstacles so that we have $k+1$ intersecting obstacles.
Let the graph search be at a vertex $v'$ with a corresponding removable set $R(v')$. 
For any $v$ adjacent to $v'$ with $c = |I(v)|$, $1 \leq c\leq k+1$, the removable set $R(v) = R(v') \cup I(v)$. 
Clearly, it holds that $|R(v')|\leq |R(v)| \leq |R(v')| + c$. 
This is due to the fact that $R(v)$ depends on the path history, that is the vertices encountered while reaching $v$. 
In order to make this more explicit, let us assume $I(v) = \{o_1, \ldots, o_k\}$. 
If $R(v') \cap I(v) = \{\emptyset\}$, the path from start to $v'$ has not encountered any of the obstacles contributing to $v$, and therefore $|R(v)| = |R(v')| + k + 1$. 
However, if $R(v') \cap I(v) \neq \{\emptyset\}$, then the path from start to $v'$ did encounter some (or all) of the obstacles in $I(v)$, and it holds that $|R(v)| < |R(v')| + k + 1$. 
This can be easily visualized in Fig.~\ref{fig:C2}. 
If we let $I(v) = \{o_1, o_3\}$, and if $I(v) \subseteq R(v')$, then $R(v) = R(v')$. 
If either $o_1 \in R(v')$ or $o_3 \in R(v')$, we then have $|R(v)| = |R(v')| + 1$. 
Finally, if neither $o_1 \not\in R(v')$ nor $o_3 \not\in R(v')$, we have that $|R(v)| = |R(v')| + 2$.  

As pointed out before, MCR can be solved exactly by expanding the corresponding graph in the increasing order of the cardinalities $|R(v)|$. 
However, as argued above, $R(v)$ is path-dependent. 
This means that at each node $v$, $|R(v) \backslash R(v')|$ could take different values depending on whether obstacles in $I(v)$ previously appeared on the path till reaching $v$. 
In other words, the state $(v, R(v))$ could be reached by many paths.
\begin{lemma}
If each obstacle is simply connected and intersects no more than $k$ obstacles, a vertex $v$ with $|I(v)| = k + 1$ and a removable set $R(v)$ could possibly be reached by $2^{k+1}$ different paths. 
\label{lemma1}
\end{lemma}
\begin{proof}
Let us assume that the graph search is currently at a vertex $v'$ with a removable set $R(v')$. 
For any $v$ adjacent to $v'$ with $I(v) = k + 1$, it can be inferred that
\begin{itemize}
\item if any one element of $I(v)$ belongs to $R(v')$, then $|R(v)| = |R(v')| + k$, which is possible in $\binom{k+1}{1}$ ways;
\item if any two elements of $I(v)$ belong to $R(v')$, then $|R(v)| = |R(v')| + k - 1$, which is possible in $\binom{k+1}{2}$ ways;\\
{$\vdots$}
\vspace{0.2cm}
\item if any $k$ elements of $I(v)$ belong to $R(v')$, then $|R(v)| = |R(v')| + 1$, which is possible in $\binom{k+1}{k}$ ways.
\end{itemize}
Finally, it holds that $|R(v)| = |R(v')| + k + 1$ if $I(v) \nsubseteq R(v')$, and $R(v) = R(v')$ if $I(v) \subseteq R(v')$, which could be mapped to $\binom{k+1}{0}$ and $\binom{k+1}{k+1}$, respectively. 

If we add all the quantities above, we obtain the different possible paths to arrive at $(v, R(v))$. 
An exact expression can be obtained using the binomial sum $\sum_{i=0}^{k+1}\binom{k+1}{i} = 2^{k+1}$. 
Therefore, the state $(v, R(v))$ could potentially be reached via $2^{k+1}$ different paths.
\end{proof} 

\noindent We will now prove that MCR with simply connected obstacles with no more than $k+1$ obstacle intersections cannot be solved in polynomial time.
\begin{theorem}
MCR with simply connected obstacles such that each obstacle intersects no more than $k$ obstacles, with $k$ held constant, cannot be solved in polynomial time.
\label{theorem1}
\end{theorem}
\begin{proof}
Let there be $n$ \textit{movable} obstacles such that each obstacle is simply connected and intersects no more than $k$ obstacles. 
Since there are $2^n$ possible subsets of obstacles, each vertex could appear in $2^n$ search states. 
From Lemma~\ref{lemma1}, we know that a vertex $v$ with a given subset of obstacles (removable set $R(v)$) could be reached via $2^{k+1}$ possible paths.
Therefore, each vertex could potentially appear in $2^{n+k+1}$ search states. 
Consequently, in the worst case, the exact search generates $O(|E|2^{n+k+1})$ search states. 
Therefore, MCR with simply connected obstacles such that each obstacle intersects no more than $k$ obstacles cannot be solved in polynomial time.
\end{proof}

%% file: fixed.tex
\subsection{Considered Sub-classes}

Though exact MCR via graph search is worst-case exponential time, one may ask if there exist sub-classes of MCR problems that are solvable in polynomial time. 
In this Section, we study one such sub-class. 

For an instance of MCR, let us consider the minimum set of obstacles $S^{\star}$ that must be removed to obtain a feasible path, that is the optimal solution. 
As seen before, the NP-hardness of MCR is caused by overlapping obstacles since they make $R(v)$ path-dependent\footnote{Exact MCR using best-first search for non-overlapping and simply connected obstacles could still generate $O(|E|2^n)$ states. 
However, for such sub-classes, a greedy search (running in $O(|E|n)$) generates an optimal solution~\cite{hauser2014IJRR}.}. 
We have also seen that the worst-case exponential complexity is attributed to the different possible combinations of obstacle sets. 
Since the exact search proceeds in the order of increasing cardinality $|R(v)|$, a natural question arises: \textit{does there exist a fixed $\alpha \leq n$ such that for all instances of MCR with $|S^{\star}| < \alpha$, an optimal solution can be found in polynomial time?}
\begin{lemma}
The binomial sum of the first $k$ terms has an upper bound given by
\begin{equation}
\sum_{i=0}^{k}\binom{n}{i} \leq 2 \binom{n}{k}.
\end{equation}
\noindent whenever $k< \frac{n+1}{3}$.
\label{lemma2}
\end{lemma}
\begin{proof}
For $k< \frac{n+1}{3}$, we have
\begin{equation}
\binom{n}{i} = \frac{n-i+1}{i}\binom{n}{i-1} \geq 2\binom{n}{i-1}.
\end{equation}
Furthermore, we have
\begin{equation}
\sum_{i=1}^{k}\binom{n}{i} \geq 2\sum_{i=1}^{k} \binom{n}{i-1} \implies \binom{n}{k} \geq \binom{n}{0} + \sum_{i=0}^{k-1}\binom{n}{i}.
\end{equation}
Adding $\binom{n}{k}$ to both the sides of the inequality, we get
\begin{equation}
2\binom{n}{k} \geq \binom{n}{0} + \sum_{i=0}^{k}\binom{n}{i} \implies \sum_{i=0}^{k}\binom{n}{i} \leq 2 \binom{n}{k}.
\end{equation}
\end{proof}

In general, given an upper bound $\alpha$ on the number of obstacles that may need to be removed, we can safely prune the vertices $|I(v)|>\alpha$, since $|S^{\star}|\leq\alpha$. 
We thus get a reduced graph $G'$. 
Similarly, for the MCR search on $G'$, states $(v, R(v))$ with $|R(v)|>\alpha$ can also be pruned since such states do not affect the solution as $|S^{\star}|\leq\alpha$. 
Furthermore, a vertex could be reached via many paths and hence could produce different states. 
Therefore, a state $(v, R(v))$ can be pruned if during the search the vertex was previously reached by a state $(v, R'(v))$ such that $R(v) \supseteq R'(v)$. 
Thus, given an $\alpha$ and states pruned as argued above, the search generates $O\left(\sum_{i=0}^{\alpha}\binom{n}{i}\right)$ states. 
\begin{theorem}
Instances of MCR for simply connected obstacles in the plane with $|S^{\star}| < \alpha$, where $\alpha = (n+1)/3$, can be solved in polynomial time.
\label{theorem2}
\end{theorem} 
\begin{proof}
For any $\alpha < (n+1)/3$, the exact search generates $O(n^{\alpha})$ states. 
This directly follows from our arguments above on pruning sub-optimal states and using Lemma~\ref{lemma2}. 
Thus, MCR with $|S^{\star}|< (n+1)/3$ is in $P$. 
\end{proof}
\begin{figure}[t!]
	\centering
		\includegraphics[scale=0.5]{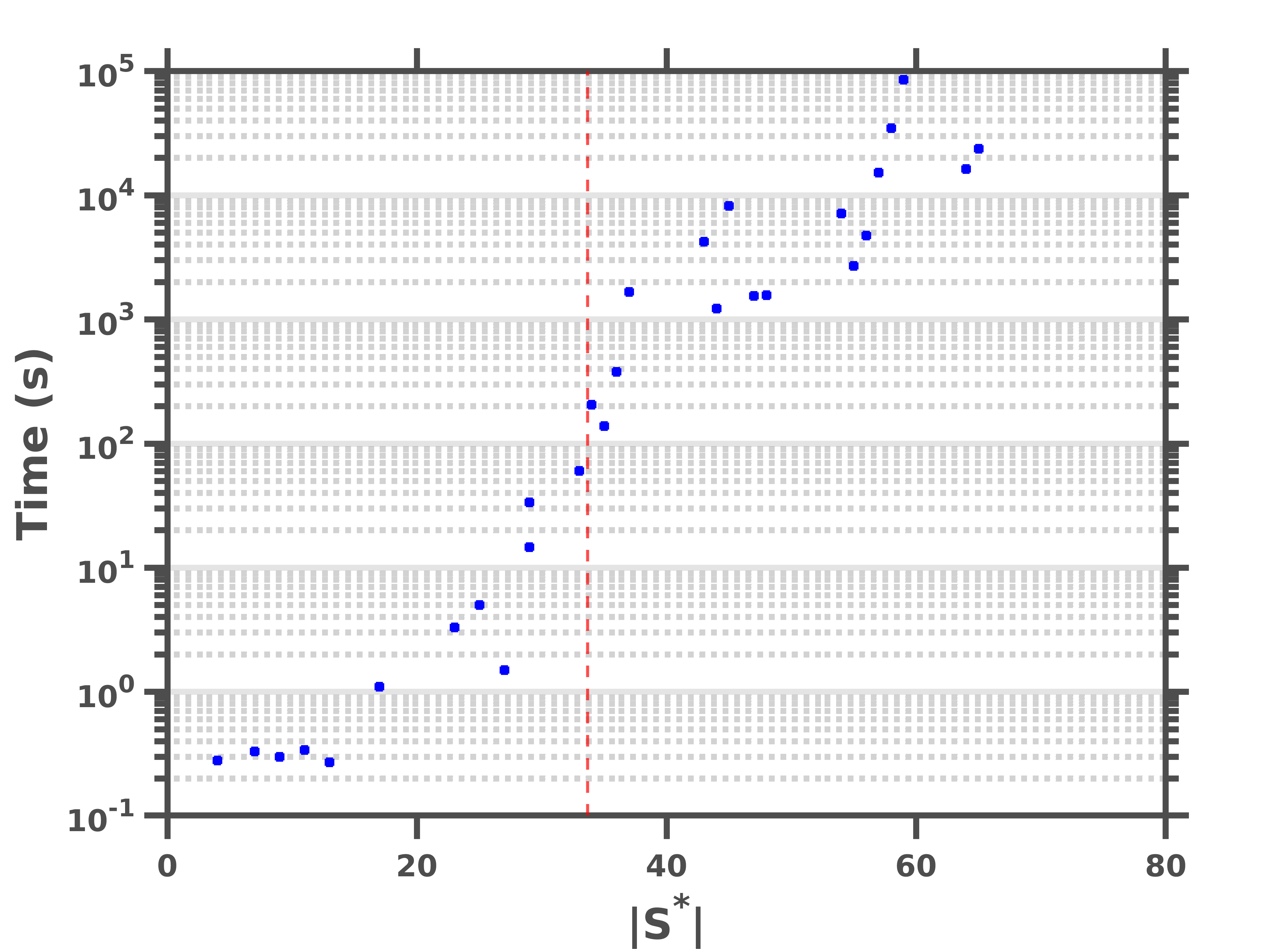}
		\caption{Running times plotted against the MCR solution (cardinality of $S^{\star}$) for different instances of 100 randomly sampled axis-parallel rectangles.}
	\label{fig:random_rectangle}
\end{figure} 
\subsection{Empirical Validation}
In order to validate our claim, we employ the exact search on different instances of MCR, in which the obstacles are axis-parallel rectangles. 
To simulate different MCR instances, we consider a $30 \times 30$ grid, and randomly sample $n$ points with each sampled point corresponding to an obstacle. 
For each sampled point we then randomly sample its length and breadth. 
Finally, for each instance, the start and goal locations are randomly chosen from the grid points. 
The tests are performed on an Intel{\small\textregistered} Core i7-10510U CPU$@$1.80GHz$\times$8 with 16GB RAM under Ubuntu 18.04 LTS.
Fig.~\ref{fig:random_rectangle} shows running times for the exact search on the $30 \times 30$ grid with $100$ obstacles. 
The red vertical line corresponds to $\alpha = (n+1)/3$. 
For $|S^{\star}|>(n+1)/3$, as the distance between the points and the $\alpha$ line increases, the exact search quickly becomes infeasible, requiring several hours to compute the optimal solution. 
It should be noted that the computational performance is greatly affected by the number of obstacle overlaps (obstacle distribution), as we have shown in the previous section. 
This justifies the non-monotonic relationship between MCR solution instances and their corresponding computational times.

\subsection{Computing an Upper Bound for $|S^{\star}|$} 
\begin{figure}[t!]
\centering
\includegraphics[scale=0.3]{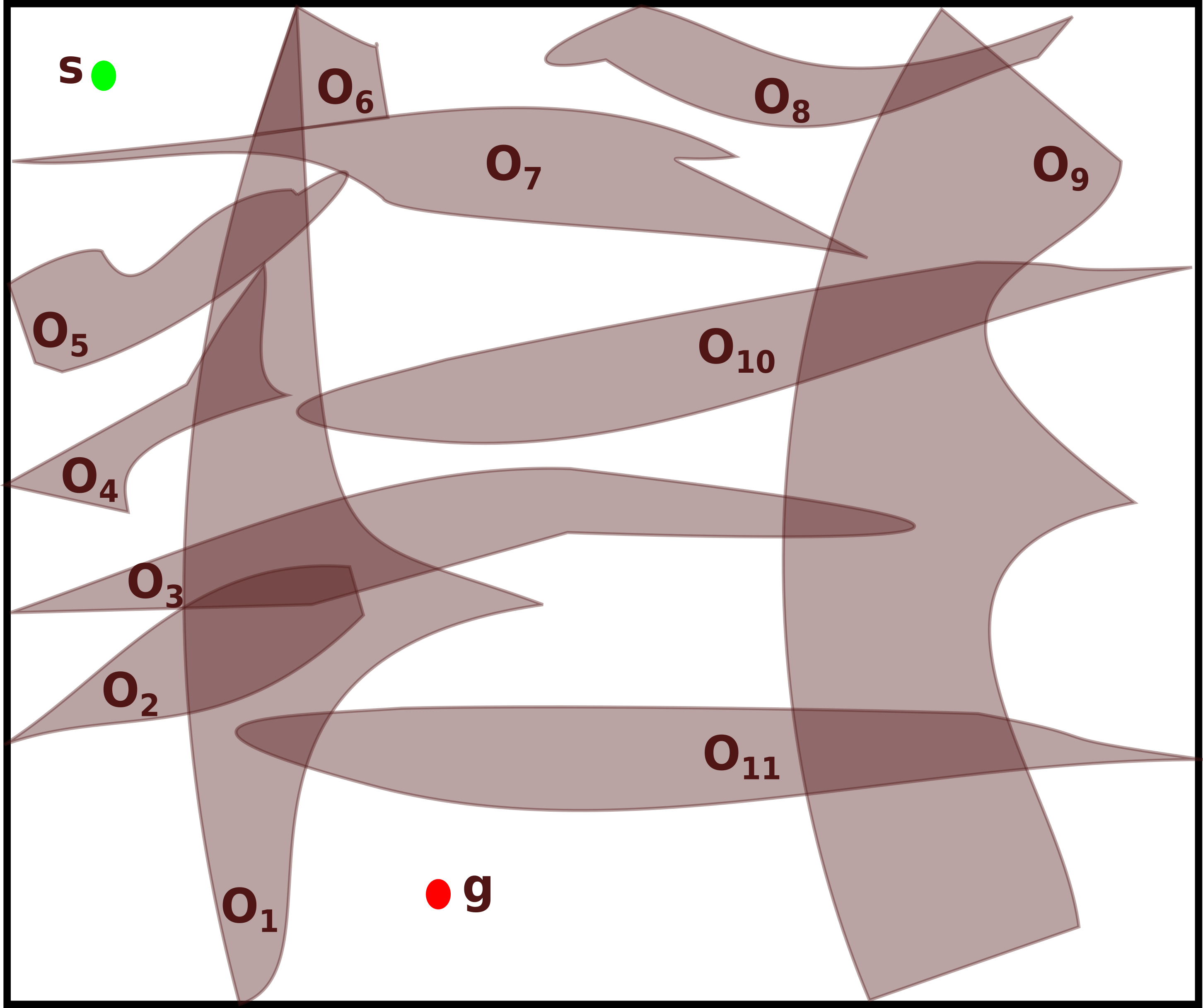}
\caption{
A 6-obstacle solution ($S = \{o_1, o_2, o_3, o_4, o_5, o_7\}$) is returned when a minimum weight path is computed with each vertex assigned a weight equal to its intersection function. 
The optimal solution $S^{\star} =\{o_1, o_3, o_7\}$ is obtained by removing $3$ obstacles.}
\label{fig:weigths}
\end{figure} 
An implication of the result in Theorem~\ref{theorem2} is that if we know that the inequality $|S^{\star}| < (n+1)/3$ holds, then the exact graph search approach can be employed. 
Such knowledge could be obtained by determining a tight upper bound (or one must hope for the existence of an \textit{oracle}). 
One straightforward approach is to solve a weighted graph with each vertex $v$ assigned a weight of $|I(v)|$. 
It is noteworthy that if the colors repeat no more than $p$ times, the minimum weight path (say, using Dijkstra’s algorithm) gives a $p-$approximate solution, where $p = \Theta(n)$. 
Fig.~\ref{fig:weigths} shows an example in which the minimum weight path is $S = \{o_1, o_2, o_3, o_4, o_5, o_7\}$. 
However the optimal MCR solution is $S^{\star} =\{o_1, o_3, o_7\}$. 
An upper bound is also easily obtained by assigning unit weight to all the vertices. 
This approach could achieve an $O(n)$ error and one such instance is displayed in Fig.~\ref{fig:equal}. 

Another natural approach is to employ the greedy search, whereby at each $v$ the minimum $|R(v)|$ subset would be expanded, and the remaining states with $|R'(v)| \geq |R(v)|$ would be pruned, which runs in $O(n)$ time. 
However, the greedy search has worst-case $O(n)$ error.
An example is shown in Fig.~\ref{fig:greedy}.
\begin{figure}[t]
\subfloat[]{\includegraphics[scale=0.43]{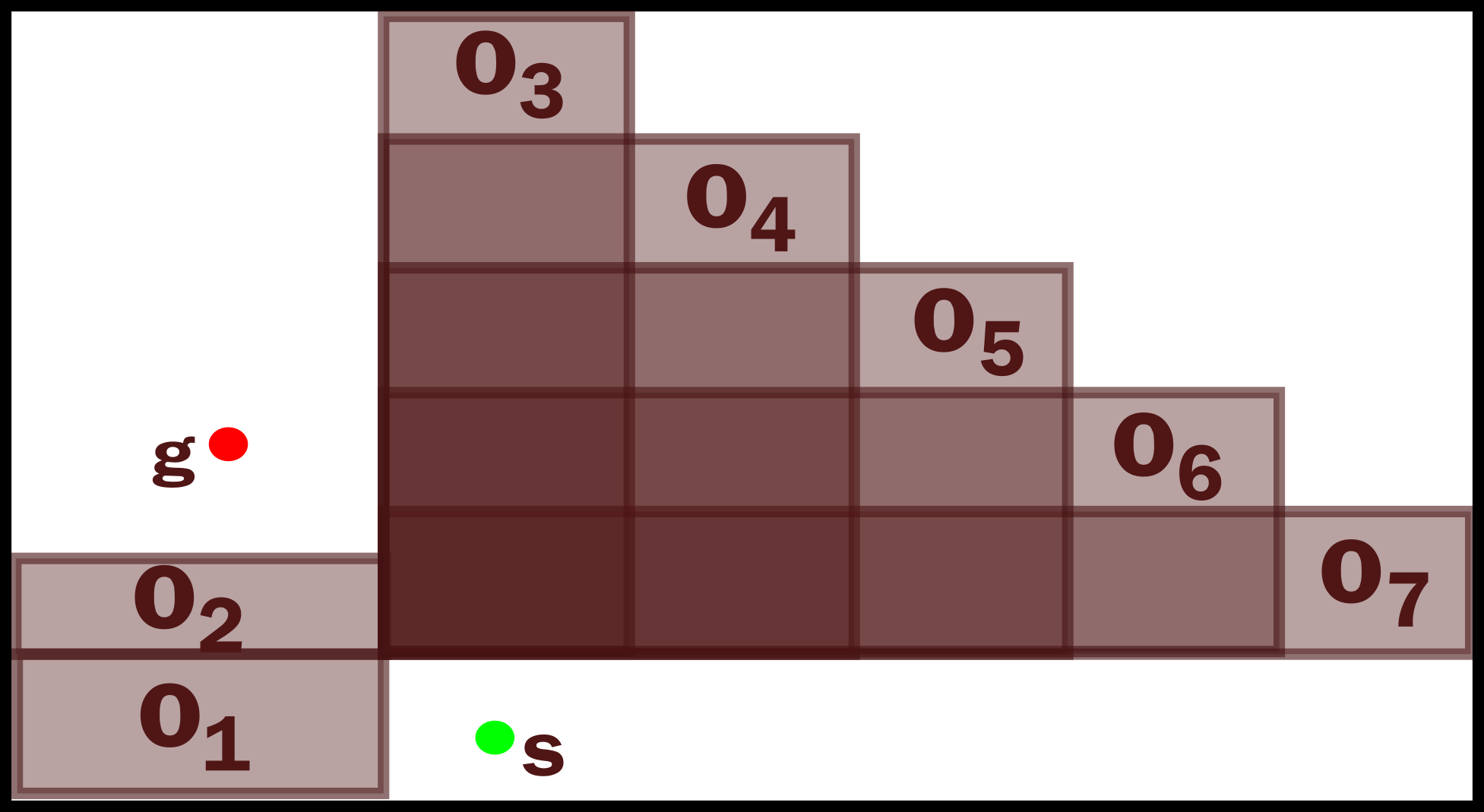}\label{fig:equal}}\hfill
\subfloat[]{\includegraphics[scale=0.5]{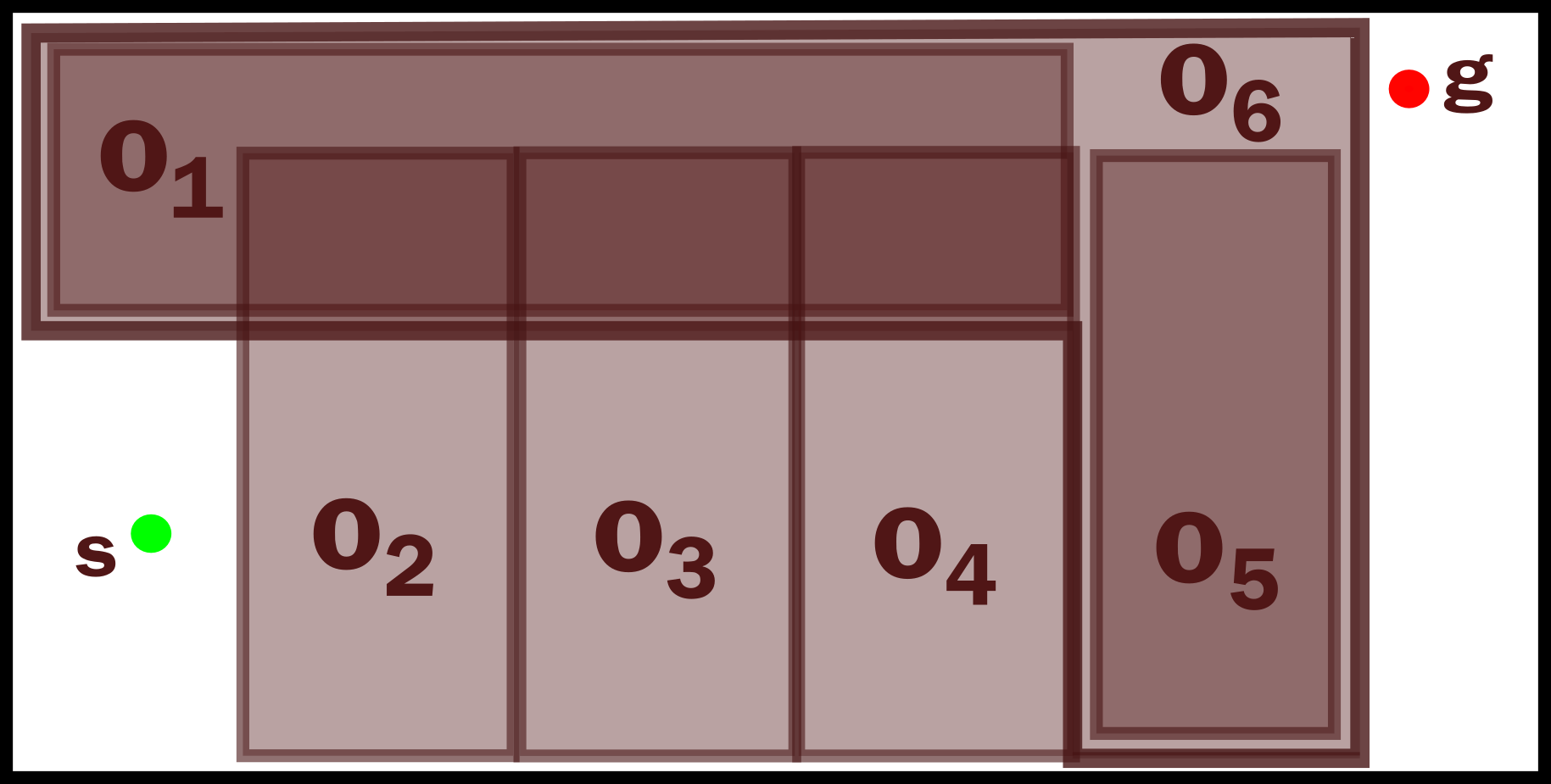}\label{fig:greedy}}\hfill
\caption{
(a) The optimal solution is $S^{\star} = \{o_1, o_2\}$. 
Assigning equal weights to each vertex, we have a sub-optimal path with $S = \{o_3, \ldots, o_7\}$.
(b) An example with five rectangular obstacles $o_1, \ldots, o_5$, and an L-shaped obstacle $o_6$ in the plane. 
The exact search gives $S^{\star} = \{o_1, o_6\}$ but the greedy search returns $S=\{o_2,o_3,o_4,o_5,o_6\}$.
}
\label{fig:shortest}
\end{figure}

Though instances achieving a tight upper bound can be easily synthesized, in general, the approaches discussed above compute conservative upper bounds. 
As seen in Fig.~\ref{fig:shortest}, even when the obstacles are axis-parallel rectangles, an $O(n)$ error is achieved. 
A $1.5$-approximation for unit disks is known due to Chan and Kirkpatrick~\cite{chan2014TCS}.
Do tighter bounds exist for different sub-classes? 
Do polynomial-time algorithms exist for MCR that give an $\epsilon$-approximation ($\epsilon>0$ and $\epsilon \to 0$)? 
We leave them as open theoretical questions.